\documentclass{article}
\usepackage[T1]{fontenc}
\usepackage[utf8]{inputenc}
\usepackage{amsmath}
\usepackage{amsthm}
\usepackage{graphicx}
\usepackage[unicode=true,
 bookmarks=false,
 breaklinks=false,pdfborder={0 0 1},backref=section,colorlinks=false]
 {hyperref}
\usepackage{breakurl}
\usepackage[numbers,sort&compress]{natbib}

\makeatletter
\numberwithin{equation}{section}
\numberwithin{figure}{section}
\theoremstyle{plain}
\newtheorem{thm}{\protect\theoremname}
  \theoremstyle{plain}
  \newtheorem{lem}[thm]{\protect\lemmaname}
  \theoremstyle{plain}
  \newtheorem{fact}[thm]{\protect\factname}
  \theoremstyle{plain}
  \newtheorem{cor}[thm]{\protect\corollaryname}
  \theoremstyle{plain}
  \newtheorem*{lem*}{\protect\lemmaname}
  \theoremstyle{plain}
  \newtheorem*{thm*}{\protect\theoremname}

 \usepackage[final]{nips_2016_modified}

\usepackage[utf8]{inputenc} 
\usepackage[T1]{fontenc}    
\usepackage{hyperref}       
\usepackage{url}            
\usepackage{booktabs}       
\usepackage{amsfonts}       
\usepackage{amsmath}
\usepackage{amssymb}
\usepackage{nicefrac}       
\usepackage{microtype}      

\author{
  Daniel Soudry \\
  Department of Statistics \\
  Columbia University \\
  New York, NY 10027, USA\\
  \texttt{daniel.soudry@gmail.com} \\
\And
   Yair Carmon \\
   Department of Electrical Engineering \\
   Stanford University \\
   Stanford, CA 94305, USA \\
  \texttt{yairc@stanford.edu} \\
}

\makeatother

  \providecommand{\corollaryname}{Corollary}
  \providecommand{\factname}{Fact}
  \providecommand{\lemmaname}{Lemma}
  \providecommand{\theoremname}{Theorem}
\providecommand{\theoremname}{Theorem}

\begin{document}
\global\long\def\E{\hat{\mathbb{E}}}
\global\long\def\wb{{\mathbf{w}}}
\global\long\def\eps{{\varepsilon}}
\global\long\def\dS{{\delta_{S}}}

\title{No bad local minima: Data independent training error guarantees for
multilayer neural networks }
\maketitle
\begin{abstract}
We use smoothed analysis techniques to provide guarantees on the training
loss of Multilayer Neural Networks (MNNs) at differentiable local
minima. Specifically, we examine MNNs with piecewise linear activation
functions, quadratic loss and a single output, under mild over-parametrization.
We prove that for a MNN with one hidden layer, the training error
is zero at every differentiable local minimum, for almost every dataset
and dropout-like noise realization. We then extend these results to
the case of more than one hidden layer. Our theoretical guarantees
assume essentially nothing on the training data, and are verified
numerically. These results suggest why the highly non-convex loss
of such MNNs can be easily optimized using local updates (e.g., stochastic
gradient descent), as observed empirically.
\end{abstract}

\section{Introduction}

Multilayer Neural Networks (MNNs) have achieved state-of-the-art performances
in many areas of machine learning \cite{LeCun2015}. This success
is typically achieved by training complicated models, using a simple
stochastic gradient descent (SGD) method, or one of its variants.
However, SGD is only guaranteed to converge to critical points in
which the gradient of the expected loss is zero \cite{Bottou1998},
and, specifically, to stable local minima \cite{Pemantle1990} (this
is true also for regular gradient descent \cite{Lee2016}). Since
loss functions parametrized by MNN weights are non-convex, it has
long been a mystery why does SGD work well, rather than converging
to ``bad'' local minima, where the training error is high (and thus
also the test error is high). 

Previous results (section \ref{sec:Related-work}) suggest that the
training error at all local minima should be low, if the MNNs have
extremely wide layers. However, such wide MNNs would also have an
extremely large number of parameters, and serious overfitting issues.
Moreover, current state of the art results are typically achieved
by deep MNNs \cite{Krizhevsky2012,He2015}, rather then wide. Therefore,
we are interested to provide training error guarantees at a more practical
number of parameters.

As a common rule-of-the-thumb, a multilayer neural network should
have at least as many parameters as training samples, and use regularization,
such as dropout \cite{Hinton2012} to reduce overfitting. For example,
Alexnet \cite{Krizhevsky2012} had 60 million parameters and was trained
using 1.2 million examples. Such over-parametrization regime continues
in more recent works, which achieve state-of-the-art performance with
very deep networks \cite{He2015}. These networks are typically under-fitting
\cite{He2015}, which suggests that the training error is the main
bottleneck in further improving performance.

In this work we focus on MNNs with a single output and leaky rectified
linear units. We provide a guarantee that the training error is zero
in every differentiable local minimum (DLM), under mild over-parametrization,
and essentially for every data set. With one hidden layer (Theorem
\ref{thm: committee machine}) we show that the training error is
zero in all DLMs, whenever the number of weights in the first layer
is larger then the number of samples $N$, \emph{i.e.}, when $N\leq d_{0}d_{1}$,
where $d_{l}$ is the width of the activation $l$-th layer. For MNNs
with $L\geq3$ layers we show that, if $N\leq d_{L-2}d_{L-1}$, then
convergence to potentially bad DLMs (in which the training error is
not zero) can be averted by using a small perturbation to the MNN's
weights and then fixing all the weights except the last two weight
layers (Corollary \ref{cor: MNN instability}).

A key aspect of our approach is the presence of a multiplicative dropout-like
noise term in our MNNs model. We formalize the notion of validity
for essentially every dataset by showing that our results hold almost
everywhere with respect to the Lebesgue measure over the data and
this noise term. This approach is commonly used in smoothed analysis
of algorithms, and often affords great improvements over worst-case
guarantees (\emph{e.g.}, \cite{Spielman2009}). Intuitively, there
may be some rare cases where our results do not hold, but almost any
infinitesimal perturbation of the input and activation functions will
fix this. Thus, our results assume essentially no structure on the
input data, and are unique in that sense.

\section{Related work\label{sec:Related-work}}

At first, it may seem hopeless to find any training error guarantee
for MNNs. Since the loss of MNNs is highly non-convex, with multiple
local minima \cite{Fukumizu2000}, it seems reasonable that optimization
with SGD would get stuck at some bad local minimum. Moreover, many
theoretical hardness results (reviewed in \cite{Sima2002}) have been
proven for MNNs with one hidden layer. 

Despite these results, one can easily achieve zero training error
\cite{Baum1988,Nilsson1965}, if the MNN's last hidden layer has more
units than training samples ($d_{l}\geq N$). This case is not very
useful, since it results in a huge number of weights (larger than
$d_{l-1}N$ ), leading to strong over-fitting. However, such wide
networks are easy to optimize, since by training the last layer we
get to a global minimum (zero training error) from almost every random
initialization \cite{Huang2006,Livni2014,Haeffele2015}. 

Qualitatively similar training dynamics are observed also in more
standard (narrower) MNNs. Specifically, the training error usually
descends on a single smooth slope path with no ``barriers''\cite{Goodfellow2014},
and the training error at local minima seems to be similar to the
error at the global minimum \cite{Dauphin2014}. The latter was explained
in \cite{Dauphin2014} by an analogy with high-dimensional random
Gaussian functions, in which any critical point high above the global
minimum has a low probability to be a local minimum. A different explanation
to the same phenomenon was suggested by \cite{Choromanska2014}. There,
a MNN was mapped to a spin-glass Ising model, in which all local minima
are limited to a finite band above the global minimum. 

However, it is not yet clear how relevant these statistical mechanics
results are for actual MNNs and realistic datasets. First, the analogy
in \cite{Dauphin2014} is qualitative, and the mapping in \cite{Choromanska2014}
requires several implausible assumptions (\emph{e.g.}, independence
of inputs and targets). Second, such statistical mechanics results
become exact in the limit of infinite parameters, so for a finite
number of layers, each layer should be infinitely wide. However, extremely
wide networks may have serious over-fitting issues, as we explained
before.

Previous works have shown that, given several limiting assumptions
on the dataset, it is possible to get a low training error on a MNN
with one hidden layer: \cite{Gori1992} proved convergences for linearly
separable datasets; \cite{Safran2015} either required that $d_{0}>N$,
or clustering of the classes. Going beyond training error, \cite{Andoni2014}
showed that MNNs with one hidden layer can learn low order polynomials,
under a product of Gaussians distributional assumption on the input.
Also, \cite{Janzamin2015} devised a tensor method, instead of the
standard SGD method, for which MNNs with one hidden layer are guaranteed
to approximate arbitrary functions. Note, however, the last two works
require a rather large $N$ to get good guarantees.

\section{Preliminaries \label{sec:Preliminaries}}

\paragraph{Model.}

We examine a Multilayer Neural Network (MNN) optimized on a finite
training set $\left\{ \mathbf{x}^{\left(n\right)},y^{\left(n\right)}\right\} _{n=1}^{N}$,
where $\mathbf{X}\triangleq\left[\mathbf{x}^{\left(1\right)},\dots,\mathbf{x}^{\left(N\right)}\right]\in\mathbb{R}^{d_{0}\times N}$
are the input patterns, $\left[y^{\left(1\right)},\dots,y^{\left(N\right)}\right]\in\mathbb{R}^{1\times N}$
are the target outputs (for simplicity we assume a scalar output),
and $N$ is the number of samples. The MNN has $L$ layers, in which
the layer inputs $\mathbf{u}_{l}^{\left(n\right)}\in\mathbb{R}^{d_{l}}$
and outputs $\mathbf{v}_{l}^{\left(n\right)}\in\mathbb{R}^{d_{l}}$
(a component of $\mathbf{v}_{l}$ is denoted $v_{i,l}$) are given
by
\begin{equation}
\forall n,\forall l\text{\ensuremath{\geq}}1:\,\begin{array}{c}
\mathbf{u}_{l}^{\left(n\right)}\triangleq\mathbf{W}_{l}\mathbf{v}_{l-1}^{\left(n\right)}\,;\,\mathbf{v}_{l}^{\left(n\right)}\triangleq\mathrm{diag}\left(\boldsymbol{a}_{l}^{\left(n\right)}\right)\mathbf{u}_{l}^{\left(n\right)}\,\end{array}\label{eq: MNN}
\end{equation}
where \textbf{$\mathbf{v}_{0}^{\left(n\right)}=\mathbf{x}^{\left(n\right)}$}
is the input of the network, $\mathbf{W}_{l}\in\mathbb{R}^{d_{l}\times d_{l-1}}$
are the weight matrices (a component of $\mathbf{W}_{l}$ is denoted
$W_{ij,l}$, bias terms are ignored for simplicity), and $\boldsymbol{a}_{l}^{\left(n\right)}=\boldsymbol{a}_{l}^{\left(n\right)}(\mathbf{u}_{l}^{\left(n\right)})$
are piecewise constant activation slopes defined below. We set $\mathbf{A}_{l}\triangleq[\boldsymbol{a}_{l}^{\left(1\right)},\dots,\boldsymbol{a}_{l}^{\left(N\right)}]$.

\paragraph{Activations.}

Many commonly used piecewise-linear activation functions (\emph{e.g.},\emph{
}rectified linear unit, maxout, max-pooling) can be written in the
matrix product form in eq. \eqref{eq: MNN}. We consider the following
relationship:
\[
\forall n:\,a_{L}^{\left(n\right)}=1,\,\forall l\leq L-1\,:\,a_{i,l}^{\left(n\right)}(\mathbf{u}_{l}^{\left(n\right)})\triangleq\epsilon_{i,l}^{\left(n\right)}\cdot\begin{cases}
1 & ,\,\mathrm{if}\,u_{i,l}^{\left(n\right)}\geq0\\
s & ,\,\mathrm{if}\,u_{i,l}^{\left(n\right)}<0
\end{cases}\,\,.
\]
When $\boldsymbol{\mathcal{E}}{}_{l}\triangleq[\boldsymbol{\epsilon}_{l}^{\left(1\right)},\dots,\mathbf{\boldsymbol{\epsilon}}_{l}^{\left(N\right)}]=\mathbf{1}$
we recover the common leaky rectified linear unit (leaky ReLU) nonlinearity,
with some fixed slope $s\neq0$. The matrix $\boldsymbol{\mathcal{E}}{}_{l}$
can be viewed as a realization of dropout noise — in most implementations
$\epsilon_{i,l}^{\left(n\right)}$ is distributed on a discrete set
(\emph{e.g.}, $\{0,1\}$), but competitive performance is obtained
with continuous distributions (\emph{e.g.} Gaussian) \cite{Srivastava2014a,Xu2015}.
Our results apply directly to the latter case. The inclusion of $\boldsymbol{\mathcal{E}}{}_{l}$
is the innovative part of our model — by performing smoothed analysis
jointly on $\mathbf{{\bf X}}$ and $\left(\boldsymbol{\mathcal{E}}_{1},...,\boldsymbol{\mathcal{E}}_{L-1}\right)$
we are able to derive strong training error guarantees. However, our
use of dropout is purely a proof strategy; we never expect dropout
to reduce the training error in realistic datasets. This is further
discussed in sections \ref{sec:Numerical-Experiments} and \ref{sec:Discussion}.

\paragraph{Measure-theoretic terminology}

Throughout the paper, we make extensive use of the term $\left(\mathbf{C}_{1},....,\mathbf{C}_{k}\right)$-almost
everywhere, or a.e. for short. This is taken to mean, almost everywhere
with respect of the Lebesgue measure on all of the entries of $\mathbf{C}_{1},....,\mathbf{C}_{k}$.
A property hold a.e. with respect to some measure, if the set of objects
for which it doesn't hold has measure 0. In particular, our results
hold with probability 1 whenever $\left(\boldsymbol{\mathcal{E}}_{1},...,\boldsymbol{\mathcal{E}}_{L-1}\right)$
is taken to have i.i.d. Gaussian entries, and arbitrarily small Gaussian
i.i.d. noise is used to smooth the input ${\bf X}$.

\paragraph{Loss function. }

We denote $e\triangleq v_{L}-y$ as the output error, where $v_{L}$
is output of the neural network with $\mathbf{v}_{0}=\mathbf{x}$,
\textbf{$\mathrm{\boldsymbol{e}}=\left[e^{\left(1\right)},\dots,e^{\left(N\right)}\right]$},
and $\E$ as the empirical expectation over the training samples.
We use the mean square error, which can be written as one of the following
forms
\begin{equation}
\mathrm{MSE}\triangleq\frac{1}{2}\E e^{2}=\frac{1}{2N}\sum_{n=1}^{N}\left(e^{\left(n\right)}\right)^{2}=\frac{1}{2N}\left\Vert \mathbf{e}\right\Vert ^{2},\label{eq: MSE}
\end{equation}
 The loss function depends on $\mathbf{X}$, $\left(\boldsymbol{\mathcal{E}}_{1},...,\boldsymbol{\mathcal{E}}_{L-1}\right)$,
and on the entire weight vector $\mathbf{w}\triangleq\left[\mathbf{w}_{1}^{\top},\dots,\mathbf{w}_{L}^{\top}\right]^{\top}\in\mathbb{R}^{\omega}$,
where $\mathbf{w}_{l}\triangleq\mathrm{vec}\left(\mathbf{W}_{l}\right)$
is the flattened weight matrix of layer $l$, and $\omega=\sum_{l=1}^{L}d_{l-1}d_{l}$
is total number of weights. 

\section{Single Hidden layer \label{sec: Single hidden layer}}

MNNs are typically trained by minimizing the loss over the training
set, using Stochastic Gradient Descent (SGD), or one of its variants
(\emph{e.g.}, ADAM \cite{Kingma2015}). In this section and the next,
we guarantee zero training loss in the common case of an over-parametrized
MNN. We do this by analyzing the properties of differentiable local
minima (DLMs) of the MSE (eq. \eqref{eq: MSE}). We focus on DLMs,
since under rather mild conditions \cite{Pemantle1990,Bottou1998},
SGD asymptotically converges to DLMs of the loss (for finite $N$,
a point can be non-differentiable only if $\exists i,l,n\text{ such that }u_{i,l}^{(n)}=0$). 

We first consider a MNN with one hidden layer $\left(L=2\right)$.
We start by examining the MSE at a DLM 
\begin{equation}
\frac{1}{2}\E e^{2}=\frac{1}{2}\E\left(y-\mathbf{W}_{2}\mathrm{diag}\left(\boldsymbol{a}_{1}\right)\mathbf{W}_{1}\mathbf{x}\right)^{2}=\frac{1}{2}\E\left(y-\boldsymbol{a}_{1}^{\top}\mathrm{diag}\left(\mathbf{w}_{2}\right)\mathbf{W}_{1}\mathbf{x}\right)^{2}\,.\label{eq:MSE 2-layer}
\end{equation}
To simplify notation, we absorb the redundant parameterization of
the weights of the second layer into the first $\tilde{\mathbf{W}}_{1}=\mathrm{diag}\left(\mathbf{w}_{2}\right)\mathbf{W}_{1}$,
obtaining
\begin{equation}
\frac{1}{2}\E e^{2}=\frac{1}{2}\E\left(y-\boldsymbol{a}_{1}^{\top}\tilde{\mathbf{W}}_{1}\mathbf{x}\right)^{2}\,.\label{eq:MSE 2-layer - reduced}
\end{equation}

Note this is only a simplified notation — we do not actually change
the weights of the MNN, so in both equations the activation slopes
remain the same,\emph{ i.e.}, $\boldsymbol{a}_{1}^{\left(n\right)}=\boldsymbol{a}_{1}(\mathbf{W}_{1}\mathbf{x}^{\left(n\right)})\neq\boldsymbol{a}_{1}(\tilde{\mathbf{W}}_{1}\mathbf{x}^{\left(n\right)})$.
If there exists an infinitesimal perturbation which reduces the MSE
in eq. \eqref{eq:MSE 2-layer - reduced}, then there exists a corresponding
infinitesimal perturbation which reduces the MSE in eq. \eqref{eq:MSE 2-layer}.
Therefore, if $\left(\mathbf{W}_{1},\mathbf{W}_{2}\right)$ is a DLM
of the MSE in eq. \eqref{eq:MSE 2-layer}, then $\tilde{\mathbf{W}}_{1}$
must also be a DLM of the MSE in eq. \eqref{eq:MSE 2-layer - reduced}.
Clearly, both DLMs have the same MSE value. Therefore, we will proceed
by assuming that $\tilde{\mathbf{W}}_{1}$ is a DLM of eq. \eqref{eq:MSE 2-layer - reduced},
and any constraint we will derive for the MSE in eq. \eqref{eq:MSE 2-layer - reduced}
will automatically apply to any DLM of the MSE in eq. \eqref{eq:MSE 2-layer}.

If we are at a DLM of eq. \eqref{eq:MSE 2-layer - reduced}, then
its derivative is equal to zero. To calculate this derivative we rely
on two facts. First, we can always switch the order of differentiation
and expectation, since we average over a finite training set. Second,
at any a differentiable point (and in particular, a DLM), the derivative
of $\boldsymbol{a}_{1}$ with respect to the weights is zero. Thus,
we find that, at any DLM, 
\begin{equation}
\nabla_{\tilde{\mathbf{W}}_{1}}\mathrm{MSE}=\E\left[e\boldsymbol{a}_{1}\mathbf{x}^{\top}\right]=0\,.\label{eq:linear regression}
\end{equation}
To reshape this gradient equation to a more convenient form, we denote
Kronecker's product by $\otimes$, and define the ``gradient matrix''
(without the error $e$)
\begin{equation}
\mathbf{G}_{1}\triangleq\mathbf{A}_{1}\circ\mathbf{X}\triangleq\left[\boldsymbol{a}_{1}^{\left(1\right)}\otimes\mathbf{x}^{\left(1\right)},\dots,\boldsymbol{a}_{1}^{\left(N\right)}\otimes\mathbf{x}^{\left(N\right)}\right]\in\mathbb{R}^{d_{0}d_{1}\times N}\,,\label{eq:G}
\end{equation}
where $\circ$ denotes the Khatari-Rao product (\emph{cf. }\cite{Allman2009},
\cite{Bhaskara2013}). Using this notation, and recalling that \textbf{$\mathrm{\mathbf{e}}=\left[e^{\left(1\right)},\dots,e^{\left(N\right)}\right],$}
eq. \eqref{eq:linear regression} becomes 
\begin{equation}
\mathbf{G}_{1}\mathbf{e}=0\,.\label{eq: orthogonality principle}
\end{equation}
Therefore, $\mathbf{e}$ lies in the right nullspace of $\mathbf{G}_{1}$,
which has dimension $N-\mathrm{rank}\left(\mathbf{G}_{1}\right)$.
Specifically, if $\mathrm{rank}\left(\mathbf{G}_{1}\right)=N$, the
only solution is $\mathbf{e}=0$. This immediately implies the following
lemma.
\begin{lem}
\label{lem: full rank (G)}Suppose we are at some DLM of of eq. \eqref{eq:MSE 2-layer - reduced}.
If $\mathrm{rank}\left(\mathbf{G}_{1}\right)=N$, then $\mathrm{MSE}=0$.
\end{lem}
To show that $\mathbf{G}_{1}$ has, generically, full column rank,
we state the following important result, which a special case of \cite[lemma 13]{Allman2009},
\begin{fact}
\label{fact:Khatri-Rao multiplies rank}For $\mathbf{B}\in\mathbb{R}^{d_{B}\times N}$
and $\mathbf{C}\in\mathbb{R}^{d_{C}\times N}$ with $N\leq d_{B}d_{C}$,
we have, $\left({\bf B},{\bf C}\right)$ almost everywhere,
\begin{equation}
\mathrm{rank}\left(\mathbf{B}\circ\mathbf{C}\right)=N\,.\label{eq: Allman result}
\end{equation}
\end{fact}
However, since $\mathbf{A}_{1}$ depends on $\mathbf{X}$, we cannot
apply eq. \eqref{eq: Allman result} directly to $\mathbf{G}_{1}=\mathbf{A}_{1}\circ\mathbf{X}$.
Instead, we apply eq. \eqref{eq: Allman result} for all (finitely
many) possible values of $\mathrm{sign}\left(\mathbf{W}_{1}\mathbf{X}\right)$
(appendix \ref{subsec:Single-hidden-layer proof}), and obtain
\begin{lem}
\label{lem:  when G full rank?}For $L=2$, if $N\leq d_{1}d_{0}$,
then simultaneously for every $\mathbf{w}$, $\mathrm{rank}\left(\mathbf{G}_{1}\right)=\mathrm{rank}\left(\mathbf{A}_{1}\circ\mathbf{X}\right)=N$,
$(\mathbf{X},\boldsymbol{\mathcal{E}}_{1})$ almost everywhere.
\end{lem}
Combining Lemma \ref{lem: full rank (G)} with Lemma \ref{lem:  when G full rank?},
we immediately have 
\begin{thm}
\label{thm: committee machine}If $N\leq d_{1}d_{0}$,\textbf{ }then
all differentiable local minima of eq. \eqref{eq:MSE 2-layer} are
global minima with $\mathrm{MSE}=0$, $(\mathbf{X},\boldsymbol{\mathcal{E}}_{1})$
almost everywhere.
\end{thm}
Note that this result is tight, in the sense that the minimal hidden
layer width $d_{1}=\left\lceil N/d_{0}\right\rceil $, is exactly
the same minimal width which ensures a MNN can implement any dichotomy
\cite{Baum1988} for inputs in general position.

\section{Multiple Hidden Layers\label{sec:General-case}}

We examine the implications of our approach for MNNs with more than
one hidden layer. To find the DLMs of a general MNN, we again need
to differentiate the MSE and equate it to zero. As in section \ref{sec: Single hidden layer},
we exchange the order of expectation and differentiation, and use
the fact that $\boldsymbol{a}_{1},...,\boldsymbol{a}_{L-1}$ are piecewise
constant. Differentiating near a DLM with respect to $\mathbf{w}_{l}$,
the vectorized version of \textbf{$\mathbf{W}_{l}$}, we obtain
\begin{equation}
\frac{1}{2}\nabla_{\mathbf{w}_{l}}\E e^{2}=\E\left[e\nabla_{\mathbf{w}_{l}}e\right]=0\label{eq: Grad MSE}
\end{equation}
To calculate $\nabla_{\mathbf{w}_{l}}e$ for to the $l$-th weight
layer, we write\footnote{For matrix products we use the convention $\prod_{k=1}^{K}{\bf M}_{k}={\bf M}_{K}{\bf M}_{K-1}\cdots{\bf M}_{2}{\bf M}_{1}$.}
its input $\mathbf{v}_{l}$ and its back-propagated ``delta'' signal
(without the error $e$) 
\begin{equation}
\mathbf{v}_{l}\triangleq\left(\prod_{m=1}^{l}\mathrm{diag}\left(\boldsymbol{a}_{m}\right)\mathbf{W}_{m}\right)\mathbf{x}\,;\,\boldsymbol{\delta}_{l}\triangleq\mathrm{diag}\left(\boldsymbol{a}_{l}\right)\prod_{m=L}^{l+1}\mathbf{W}_{m}^{\top}\mathrm{diag}\left(\boldsymbol{a}_{m}\right)\,,\label{eq: v,delta}
\end{equation}
where we keep in mind that $\boldsymbol{a}_{l}$ are generally functions
of the inputs and the weights. Using this notation we find 
\begin{equation}
\nabla_{\mathbf{w}_{l}}e=\nabla_{\mathbf{w}_{l}}\left(\prod_{m=1}^{L}\mathrm{diag}\left(\boldsymbol{a}_{m}\right)\mathbf{W}_{m}\right)\mathbf{x}=\mathbf{\boldsymbol{\delta}}_{l}^{\top}\otimes\mathbf{v}_{l-1}^{\top}\,.\label{eq:grad e}
\end{equation}
Thus, defining 
\[
\boldsymbol{\Delta}_{l}=\left[\mathbf{\boldsymbol{\delta}}_{l}^{\left(1\right)},\dots,\mathbf{\boldsymbol{\delta}}_{l}^{\left(N\right)}\right];\,\mathbf{V}_{l}=\left[\mathbf{v}_{l}^{\left(1\right)},\dots,\mathbf{v}_{l}^{\left(N\right)}\right]
\]
 we can re-formulate eq. \eqref{eq: Grad MSE} as
\begin{equation}
\mathbf{G}_{l}\mathbf{e}=0\,,\mathrm{with}\,\mathbf{G}_{l}\triangleq\boldsymbol{\Delta}_{l}\circ\mathbf{V}_{l-1}=\left[\mathbf{\boldsymbol{\delta}}_{l}^{\left(1\right)}\otimes\mathbf{v}_{l-1}^{\left(1\right)},\dots,\mathbf{\boldsymbol{\delta}}_{l}^{\left(N\right)}\otimes\mathbf{v}_{l-1}^{\left(N\right)}\right]\in\mathbb{R}^{d_{l-1}d_{l}\times N}\label{eq: Gme=00003D0}
\end{equation}
similarly to eq. \eqref{eq: orthogonality principle} the previous
section. Therefore, each weight layer provides as many linear constraints
(rows) as the number of its parameters. We can also combine all the
constraints and get 
\begin{equation}
\mathbf{G}\mathbf{e}=0\,,\,\mathrm{with}\,\mathbf{G}\triangleq\left[\mathbf{G}_{1}^{\top},\dots,\mathbf{G}_{L}^{\top}\right]^{\top}\in\mathbb{R}^{\omega\times N},\label{eq: Ge=00003D0}
\end{equation}
In which we have $\omega$ constraints (rows) corresponding to all
the parameters in the MNN. As in the previous section, if $\omega\geq N$
and $\mathrm{rank}\left(\mathbf{G}\right)=N$ we must have $\mathbf{e}=0$.
However, it is generally difficult to find the rank of $\mathbf{G}$,
since we need to find whether different $\mathbf{G}_{l}$ have linearly
dependent rows. Therefore, we will focus on the last hidden layer
and on the condition $\mathrm{rank}\left(\mathbf{G}_{L-1}\right)=N$,
which ensures $\mathbf{e}=0$, from eq. \eqref{eq: Gme=00003D0}.
However, since $\mathbf{v}_{L-2}$ depends on the weights, we cannot
use our results from the previous section, and it is possible that
$\mathrm{rank}\left(\mathbf{G}_{L-1}\right)<N$. For example, when
$\mathbf{w}=0$ and $L\geq3$, we get $\mathbf{G}=0$ so we are at
a differentiable critical point (note it $\mathbf{G}$ is well defined,
even though $\forall l,n:\,\mathbf{u}_{l}^{\left(n\right)}=0$), which
is generally not a global minimum. Intuitively, such cases seem fragile,
since if we give $\mathbf{w}$ any random perturbation, one would
expect that ``typically'' we would have $\mathrm{rank}\left(\mathbf{G}_{L-1}\right)=N$.
We establish this idea by first proving the following stronger result
(appendix \ref{subsec:General-case- proof}),
\begin{thm}
\label{thm: MNN instability}For $N\leq d_{L-2}d_{L-1}$ and fixed
values of ${\bf W}_{1},...,{\bf W}_{L-2}$, any differentiable local
minimum of the MSE (eq. \ref{eq: MSE}) as a function of $\mathbf{W}_{L-1}$
and $\mathbf{W}_{L}$, is also a global minimum, with $\mathrm{MSE}=0$,
$\left(\mathbf{X},\boldsymbol{\mathcal{E}}_{1},\dots,\boldsymbol{\mathcal{E}}_{L-1},{\bf W}_{1},...,{\bf W}_{L-2}\right)$
almost everywhere.
\end{thm}
Theorem \ref{thm: MNN instability} means that for any (Lebesgue measurable)
random set of weights of the first $L-2$ layers, every DLM with respect
to the weights of the last two layers is also a global minimum with
loss 0. Note that the condition $N\leq d_{L-2}d_{L-1}$ implies that
$\mathbf{W}_{L-1}$ has more weights then $N$ (a plausible scenario,
\emph{e.g.}, \cite{Krizhevsky2012}). In contrast, if, instead we
were only allowed to adjust the last layer of a random MNN, then
low training error can only be ensured with extremely wide layers
($d_{L-1}\geq N$, as discussed in section \ref{sec:Related-work}),
which require much more parameters ($d_{L-2}N$). 

Theorem \ref{thm: MNN instability} can be easily extended to other
types of neural networks, beyond of the basic formalism introduced
in section \ref{sec:Preliminaries}. For example, we can replace the
layers below $L-2$ with convolutional layers, or other types of architectures.
Additionally, the proof of Theorem \ref{thm: MNN instability} holds
(with a trivial adjustment) when $\boldsymbol{\mathcal{E}}_{1},...,\boldsymbol{\mathcal{E}}_{L-3}$
are fixed to have identical nonzero entries — that is, with dropout
turned off except in the last two hidden layers. The result continues
to hold even when $\boldsymbol{\mathcal{E}}_{L-2}$ is fixed as well,
but then the condition $N\leq d_{L-2}d_{L-1}$ has to be weakened
to $N\leq d_{L-1}\min_{l\leq L-2}d_{l}$.

Next, we formalize our intuition above that DLMs of deep MNNs must
have zero loss or be fragile, in the sense of the following immediate
corollary of Theorem \ref{thm: MNN instability},
\begin{cor}
\label{cor: MNN instability}For $N\leq d_{L-2}d_{L-1}$, let $\mathbf{w}$
be a differentiable local minimum of the MSE (eq. \ref{eq: MSE}).
Consider a new weight vector $\tilde{\mathbf{w}}=\mathbf{w}+\delta\text{\textbf{w}}$,
where $\delta\text{\textbf{w}}$ has i.i.d. Gaussian (or uniform)
entries with arbitrarily small variance. Then, $\left(\mathbf{X},\boldsymbol{\mathcal{E}}_{1},\dots,\boldsymbol{\mathcal{E}}_{L-1}\right)$
almost everywhere and with probability 1 w.r.t. $\delta{\bf w}$,
if $\tilde{{\bf W}}_{1},...,\tilde{{\bf W}}_{L-2}$ are held fixed,
all differentiable local minima of the MSE as a function of $\mathbf{W}_{L-1}$
and $\mathbf{W}_{L}$ are also global minima, with $\mathrm{MSE}=0$.
\end{cor}
Note that this result is different from the classical notion of linear
stability at differentiable critical points, which is based on the
analysis of the eigenvalues of the Hessian $\mathbf{H}$ of the MSE.
The Hessian can be written as a symmetric block matrix, where each
of its blocks $\mathbf{H}_{ml}\in\mathbb{R}^{d_{m-1}d_{m}\times d_{l-1}d_{l}}$
corresponds to layers $m$ and $l$. Specifically, using eq. \eqref{eq:grad e},
each block can be written as a sum of two components
\begin{equation}
\mathbf{H}_{ml}\triangleq\frac{1}{2}\nabla_{\mathbf{w}_{l}}\nabla_{\mathbf{w}_{m}^{\top}}\E e^{2}=\E\left[e\nabla_{\mathbf{w}_{l}}\nabla_{\mathbf{w}_{m}^{\top}}e\right]+\E\left[\nabla_{\mathbf{w}_{l}}e\nabla_{\mathbf{w}_{m}^{\top}}e\right]\triangleq\E\left[e\boldsymbol{\Lambda}_{ml}\right]+\frac{1}{N}\mathbf{G}_{m}\mathbf{G}_{l}^{\top}\,,\label{eq: Hessian-1}
\end{equation}
where, for $l<m$
\begin{equation}
\!\!\boldsymbol{\Lambda}_{ml}\!\triangleq\!\nabla_{\mathbf{w}_{l}}\nabla_{\mathbf{w}_{m}^{\top}}e=\!\nabla_{\mathbf{w}_{l}}\left(\mathbf{\boldsymbol{\delta}}_{m}\otimes\mathbf{v}_{m-1}\right)\!=\!\mathbf{\boldsymbol{\delta}}_{m}\otimes\left(\prod_{l^{\prime}=l+1}^{m-1}\mathrm{diag}\left(\boldsymbol{a}_{l^{\prime}}\right)\mathbf{W}_{l^{\prime}}\!\right)\mathrm{diag}\left(\boldsymbol{a}_{l}\right)\otimes\mathbf{v}_{l-1}^{\top}\label{eq: Lambda-1}
\end{equation}
while $\boldsymbol{\Lambda}_{ll}\!=0$, and $\boldsymbol{\Lambda}_{ml}=\boldsymbol{\Lambda}_{lm}^{\top}$
for $m<l$. Combining all the blocks, we get 
\[
\mathbf{H}=\E\left[e\boldsymbol{\Lambda}\right]+\frac{1}{N}\mathbf{G}\mathbf{G}^{\top}\in\mathbb{R}^{\omega\times\omega}\,.
\]
If we are at a DLM, then $\mathbf{H}$ is positive semi-definite.
If we examine again the differentiable critical point $\mathbf{w}=0$
and $L\geq3$, we see that $\mathbf{H}=0$, so it is not a strict
saddle. However, this point is fragile in the sense of Corollary \ref{cor: MNN instability}.

Interestingly, the positive semi-definite nature of the Hessian at
DLMs imposes additional constraints on the error. Note that the matrix
$\mathbf{G}\mathbf{G}^{\top}$ is symmetric positive semi-definite
of relatively small rank $\left(\leq N\right)$. However, $\E\left[e\boldsymbol{\Lambda}\right]$
can potentially be of high rank, and thus may have many negative eigenvalues
(the trace of $\E\left[e\boldsymbol{\Lambda}\right]$ is zero, so
the sum of all its eigenvalues is also zero). Therefore, intuitively,
we expect that for $\mathbf{H}$ to be positive semi-definite, $\mathbf{e}$
has to become small, generically (\emph{i.e.}, except at some pathological
points such as \textbf{$\mathbf{w}=0$}). This is indeed observed
empirically \cite[Fig 1]{Dauphin2014}. 

\section{Numerical Experiments\label{sec:Numerical-Experiments}}

\begin{figure*}[t]
\begin{centering}
\includegraphics[width=0.9\textwidth]{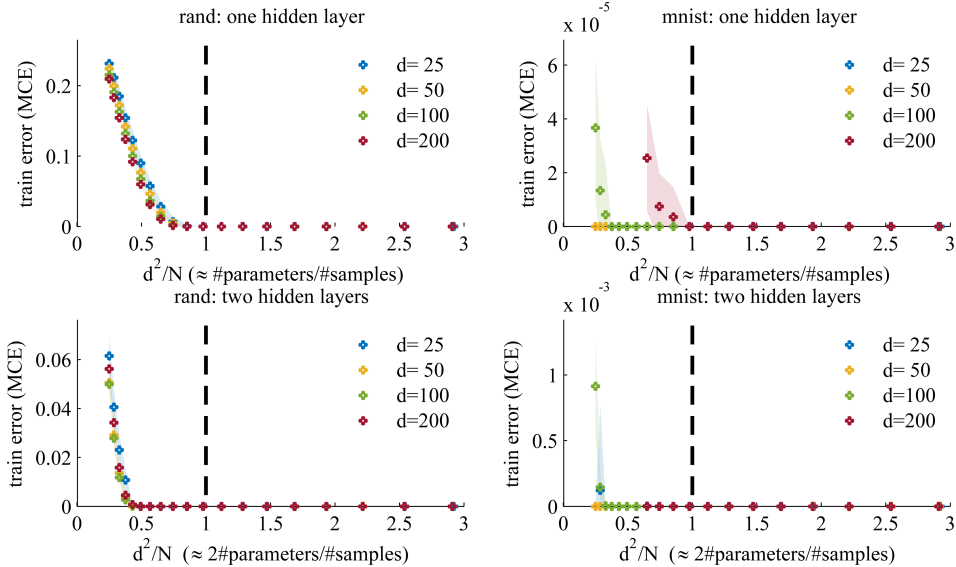}
\par\end{centering}
\caption{\textbf{Final training error (mean$\pm$std) in the over-parametrized
regime is low, as predicted by our results (right of the dashed black
line). }We trained standard MNNs with one or two hidden layers (with
widths equal to $d=d_{0}$), a single output, (non-leaky) ReLU activations,
MSE loss, and no dropout, on two datasets: (1) a synthetic random
dataset in which $\forall n=1,\dots,N$, \textbf{$\mathbf{x}^{\left(n\right)}$
}was drawn from a normal distribution $\mathcal{N}\left(0,1\right)$,\textbf{
}and $y^{\left(n\right)}=\pm1$ with probability $0.5$ (2) binary
classification (between digits $0-4$ and $5-9$) on $N$ sized subsets
of the MNIST dataset \cite{Lecun1998}. The value at a data point
is an average of the mean classification error (MCE) over 30 repetitions.
In this figure, when the mean MCE reached zero, it was zero for all
30 repetitions. \label{fig:training error}}
\end{figure*}

\begin{figure*}[t]
\begin{centering}
\includegraphics[width=0.9\textwidth]{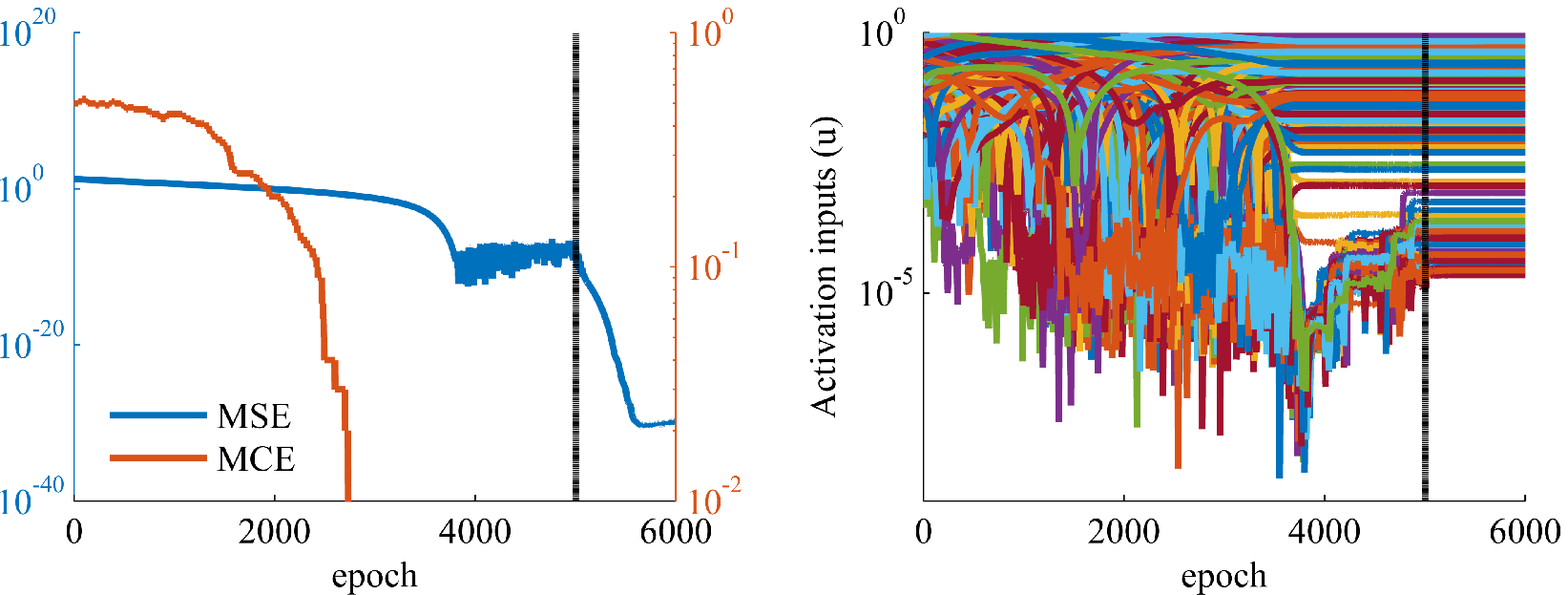}
\par\end{centering}
\caption{\textbf{The existence of differentiable local minima. }In this representative
figure, we trained a MNN with a single hidden layer, as in Fig. \ref{fig:training error},
with $d=25$, on the synthetic random data ($N=100$) until convergence
with gradient descent (so each epoch is a gradient step). Then, starting
from epoch 5000 (dashed line), we gradually decreased the learning
rate (multiplying it by $0.999$ each epoch) until it was about $10^{-9}$.
We see that the activation inputs converged to values above $10^{-5}$,
while the final MSE was about $10^{-31}$. The magnitudes of these
numbers, and the fact that all the neuronal inputs do not keep decreasing
with the learning rate, indicate that we converged to a differentiable
local minimum, with MSE equal to 0, as predicted.\label{fig:The-existence-of-DLMs}}
\end{figure*}

In this section we examine numerically our main results in this paper,
Theorems \ref{thm: committee machine} and \ref{thm: MNN instability},
which hold almost everywhere with respect to the Lebesgue measure
over the data and dropout realization. However, without dropout, this
analysis is not guaranteed to hold. For example, our results do not
hold in MNNs where all the weights are negative, so $\mathbf{A}_{L-1}$
has constant entries and therefore $\mathbf{G}_{L-1}$ cannot have
full rank.

Nonetheless, if the activations are sufficiently ``variable'' (formally,
$\mathbf{G}_{L-1}$ has full rank), then we expect our results to
hold even without dropout noise and with the leaky ReLU's replaced
with basic ReLU's ($s=0$). We tested this numerically and present
the result in Figure \ref{fig:training error}. We performed a binary
classification task on a synthetic random dataset and subsets of the
MNIST dataset, and show the mean classification error (MCE, which
is the fraction of samples incorrectly classified), commonly used
at these tasks. Note that the MNIST dataset, which contains some redundant
information between training samples, is much easier (a lower error)
than the completely random synthetic data. Thus the performance on
the random data is more representative of the ``typical worst case'',\emph{
}(\emph{i.e.}, hard yet non-pathological input), which our smoothed
analysis approach is aimed to uncover.

For one hidden layer, the error goes to zero when the number of non-redundant
parameters is greater than the number of samples ($d^{2}/N\geq1$),
as predicted by Theorem \ref{thm: committee machine}. Theorem \ref{thm: MNN instability}
predicts a similar behavior when $d^{2}/N\geq1$ for a MNN with two
hidden layers (note we trained all the layers of the MNN). This prediction
also seems to hold, but less tightly. This is reasonable, as our analysis
in section \ref{sec:General-case} suggests that typically the error
would be zero if the total number of parameters is larger the number
of training samples ($d^{2}/N\geq0.5$), though this was not proven.
We note that in all the repetitions in Figure \ref{fig:training error},
for $d^{2}\geq N$, the matrix $\mathbf{G}_{L-1}$ always had full
rank. However, for smaller MNNs than shown in Figure \ref{fig:training error}
(about $d\leq20$), sometimes $\mathbf{G}_{L-1}$ did not have full
rank. 

Recall that Theorems \ref{thm: committee machine} and \ref{thm: MNN instability}
both give guarantees only on the training error at a DLM. However,
for finite $N$, since the loss is non-differentiable at some points,
it is not clear that such DLMs actually exist, or that we can converge
to them. To check if this is indeed the case, we performed the following
experiment. We trained the MNN for many epochs, using batch gradient
steps. Then, we started to gradually decrease the learning rate. If
the we are at DLM, then all the activation inputs $u_{i,l}^{(n)}$
should converge to a distinctly non-zero value, as demonstrated in
Figure \ref{fig:The-existence-of-DLMs}. In this figure, we tested
a small MNN on synthetic data, and all the neural inputs seem to remain
constant on a non-zero value, while the MSE keeps decreasing. This
was the typical case in our experiments. However, in some instances,
we would see some $u_{i,l}^{(n)}$ converge to a very low value ($10^{-16}$).
This may indicate that convergence to non-differentiable points is
possible as well.

\paragraph{Implementation details}

Weights were initialized to be uniform with mean zero and variance
$2/d$, as suggested in \cite{He2015a}. In each epoch we randomly
permuted the dataset and used the Adam \cite{Kingma2015} optimization
method (a variant of SGD) with $\beta_{1}=0.9,\beta_{2}=0.99,\eps=10^{-8}$.
In Figure \ref{fig:training error} the training was done for no more
than $4000$ epochs  (we stopped if $\text{MCE}=0$ was reached).
Different learning rates and mini-batch sizes were selected for each
dataset and architecture. 

\section{Discussion\label{sec:Discussion}}

In this work we provided training error guarantees for mildly over-parameterized
MNNs at all differentiable local minima (DLM). For a single hidden
layer (section \ref{sec: Single hidden layer}), the proof is surprisingly
simple. We show that the MSE near each DLM is locally similar to that
of linear regression (\emph{i.e.}, a single linear neuron). This allows
us to prove (Theorem \ref{thm: committee machine}) that, almost everywhere,
if the number of non-redundant parameters $\left(d_{0}d_{1}\right)$
is larger then the number of samples $N$, then all DLMs are a global
minima with $\mathrm{MSE}=0$, as in linear regression. With more
then one hidden layers, Theorem \ref{thm: MNN instability} states
that if $N\leq d_{L-2}d_{L-1}$ (\emph{i.e.}, so $\mathbf{W}_{L-1}$
has more weights than $N$) then we can always perturb and fix some
weights in the MNN so that all the DLMs would again be global minima
with $\mathrm{MSE}=0$. 

Note that in a realistic setting, zero training error should not necessarily
be the intended objective of training, since it may encourage overfitting.
Our main goal here was to show that that essentially all DLMs provide
good training error (which is not trivial in a non-convex model).
However, one can decrease the size of the model or artificially increase
the number of samples (\emph{e.g.}, using data augmentation, or re-sampling
the dropout noise) to be in a mildly under-parameterized regime, and
have relatively small error, as seen in Figure \ref{fig:training error}.
For example, in AlexNet \cite{Krizhevsky2012} $\mathbf{W}_{L-1}$
has $4096^{2}\approx17\cdot10^{6}$ weights, which is larger than
$N=1.2\cdot10^{6}$, as required by Theorem \ref{thm: MNN instability}.
However, without data augmentation or dropout, Alexnet did exhibit
severe overfitting. 

Our analysis is non-asymptotic, relying on the fact that, near differentiable
points, MNNs with piecewise linear activation functions can be differentiated
similarly to linear MNNs \cite{Saxe2013}. We use a smoothed analysis
approach, in which we examine the error of the MNN under slight random
perturbations of worst-case input and dropout. Our experiments (Figure
\ref{fig:training error}) suggest that our results describe the typical
performance of MNNs, even without dropout. Note we do not claim that
dropout has any merit in reducing the training loss in real datasets
— as used in practice, dropout typically trades off the training performance
in favor of improved generalization. Thus, the role of dropout in
our results is purely theoretical. In particular, dropout ensures
that the gradient matrix $\mathbf{G}_{L-1}$ (eq. \eqref{eq: Gme=00003D0})
has full column rank. It would be an interesting direction for future
work to find other sufficient conditions for $\mathbf{G}_{L-1}$ to
have full column rank.

Many other directions remain for future work. For example, we believe
it should be possible to extend this work to multi-output MNNs and/or
other convex loss functions besides the quadratic loss. Our results
might also be extended for stable non-differentiable critical points
(which may exist, see section \ref{sec:Numerical-Experiments}) using
the necessary condition that the sub-gradient set contains zero in
any critical point \cite{Rockafellarf1979}. Another important direction
is improving the results of Theorem \ref{thm: MNN instability}, so
it would make efficient use of the all the parameters of the MNNs,
and not just the last two weight layers. Such results might be used
as a guideline for architecture design, when training error is a major
bottleneck \cite{He2015}. Last, but not least, in this work we focused
on the empirical risk (training error) at DLMs. Such guarantees might
be combined with generalization guarantees (\emph{e.g.}, \cite{Hardt2015a}),
to obtain novel excess risk bounds that go beyond uniform convergence
analysis.

\subsubsection*{Acknowledgments}

The authors are grateful to O. Barak, D. Carmon, Y. Han., Y. Harel,
R. Meir, E. Meirom, L. Paninski, R. Rubin, M. Stern, U. Sümbül and
A. Wolf for helpful discussions. The research was partially supported
by the Gruss Lipper Charitable Foundation, and by the Intelligence
Advanced Research Projects Activity (IARPA) via Department of Interior/
Interior Business Center (DoI/IBC) contract number D16PC00003. The
U.S. Government is authorized to reproduce and distribute reprints
for Governmental purposes notwithstanding any copyright annotation
thereon. Disclaimer: The views and conclusions contained herein are
those of the authors and should not be interpreted as necessarily
representing the official policies or endorsements, either expressed
or implied, of IARPA, DoI/IBC, or the U.S. Government.

\appendix
\newpage{}

\part*{Appendix}

In this appendix we give the proofs for our main results in the paper.
But first, we define some additional notation. Recall that for every
layer $l$, data instance $n$ and index $i$, the activation slope
$a_{i,l}^{\left(n\right)}$ takes one of two values: $\epsilon_{i,l}^{\left(n\right)}$
or $s\epsilon_{i,l}^{\left(n\right)}$(with $s\neq0$). Hence, for
a single realization of $\boldsymbol{\mathcal{E}}{}_{l}=\left[\boldsymbol{\epsilon}_{l}^{\left(1\right)},\dots,\mathbf{\boldsymbol{\epsilon}}_{l}^{\left(N\right)}\right]$,
the matrix $\mathbf{A}_{l}=\left[\boldsymbol{a}_{l}^{\left(1\right)},...,\boldsymbol{a}_{l}^{\left(N\right)}\right]\in\mathbb{R}^{d_{l}\times N}$
can have up to $2^{Nd_{l}}$ distinct values, and the tuple $({\bf A}_{1},...,{\bf A}_{L-1})$
can have at most $P=2^{N\sum_{l=1}^{L-1}d_{l}}$ distinct values.
We will find it useful to enumerate these possibilities by an index
$p\in\left\{ 1,...,P\right\} $ which will be called the activation
pattern. We will similarly denote $\mathbf{A}_{l}^{p}$ to be the
value of $\mathbf{A}_{l}$ , given under activation pattern $p$.
Lastly, we will make use of the following fact:
\begin{fact}
\label{fact:intersect-a.e.}If properties $P_{1},P_{2},...,P_{m}$
hold almost everywhere, then $\cap_{i=1}^{m}P_{i}$ also holds almost
everywhere.
\end{fact}

\section{Single hidden layer — proof of Lemma \ref{lem:  when G full rank?}
\label{subsec:Single-hidden-layer proof}}

We prove the following Lemma \ref{lem:  when G full rank?}, using
the previous notation and results from section \eqref{sec: Single hidden layer}.
\begin{lem*}
For $L=2$, if $N\leq d_{1}d_{0}$, then simultaneously for every
$\mathbf{w}$, $\mathrm{rank}\left(\mathbf{G}_{1}\right)=\mathrm{rank}\left(\mathbf{A}_{1}\circ\mathbf{X}\right)=N$,
$(\mathbf{X},\boldsymbol{\mathcal{E}}_{1})$ almost everywhere.
\end{lem*}
\begin{proof}
We fix an activation pattern $p$ and set $\mathbf{G}_{1}^{p}=\mathbf{A}_{1}^{p}\circ\mathbf{X}$.
We apply eq. \eqref{eq: Allman result}  to conclude that $\mathrm{rank}\left(\mathbf{G}_{1}^{p}\right)=\mathrm{rank}\left(\mathbf{A}_{1}^{p}\circ\mathbf{X}\right)=N$,
$(\mathbf{X},\mathbf{A}_{1}^{p})$-a.e. and hence also $(\mathbf{X},\boldsymbol{\mathcal{E}}_{1})$-a.e..
We repeat the argument for all $2^{Nd_{1}}$ values of $p$, and use
fact \ref{fact:intersect-a.e.}. We conclude that $\mathrm{rank}\left(\mathbf{G}_{1}^{p}\right)=N$
for all $p$ simultaneously, $(\mathbf{X},\boldsymbol{\mathcal{E}}_{1})$-a.e..
Since for every set of weights we have $\mathbf{G}=\mathbf{\mathbf{G}}^{p}$
for some $p$, we have $\mathrm{rank}\left(\mathbf{G}_{1}\right)=N$,
$(\mathbf{X},\boldsymbol{\mathcal{E}}_{1})$-a.e.
\end{proof}

\section{Multiple Hidden Layers — proof of theorem \ref{thm: MNN instability}\label{subsec:General-case- proof}}

First we prove the following helpful Lemma, using a technique similar
to that of \cite{Allman2009}. 
\begin{lem}
Let $\mathbf{M}\left(\boldsymbol{\theta}\right)\in\mathbb{R}^{a\times b}$
be a matrix with $a\geq b$, with entries that are all polynomial
functions of some vector $\boldsymbol{\theta}$. Also, we assume that
for some value $\boldsymbol{\theta}_{0}$, we have $\mathrm{rank}\left(\mathbf{\mathbf{M}}\left(\boldsymbol{\theta}_{0}\right)\right)=b$.
Then, for almost every $\boldsymbol{\theta}$, we have $\mathrm{rank}\left(\mathbf{\mathbf{M}}\left(\boldsymbol{\theta}\right)\right)=b$.
\label{lem: polynomial proof}
\end{lem}
\begin{proof}
There exists a polynomial mapping $g:\mathbb{\,R}^{a\times b}\to\mathbb{R}$
such that $\mathbf{M}\left(\boldsymbol{\theta}\right)$ does not have
full column rank if and only if $g\left(\mathbf{M}\left(\boldsymbol{\theta}\right)\right)=0$.
Since $b\leq a$ we can construct $g$ explicitly as the sum of the
squares of the determinants of all possible different subsets of $b$
rows from $\mathbf{M}\left(\boldsymbol{\theta}\right)$. Since $g\left(\mathbf{M}\left(\boldsymbol{\theta}_{0}\right)\right)\neq0$,
we find that $g\left(\mathbf{M}\left(\boldsymbol{\theta}\right)\right)$
is not identically equal to zero. Therefore, the zeros of such a (``proper'')
polynomial, in which $g\left(\mathbf{M}\left(\boldsymbol{\theta}\right)\right)=0$,
are a set of measure zero. 
\end{proof}
Next we prove Theorem \ref{thm: MNN instability}, using the previous
notation and the results from section \eqref{sec:General-case}:
\begin{thm*}
For $N\leq d_{L-2}d_{L-1}$ and fixed values of ${\bf W}_{1},...,{\bf W}_{L-2}$,
any differentiable local minimum of the MSE (eq. \ref{eq: MSE}) as
a function of $\mathbf{W}_{L-1}$ and $\mathbf{W}_{L}$, is also a
global minimum, with $\mathrm{MSE}=0$, $(\mathbf{X},\boldsymbol{\mathcal{E}}_{1},\dots,\boldsymbol{\mathcal{E}}_{L-1},{\bf W}_{1},...,{\bf W}_{L-2})$
almost everywhere.
\end{thm*}
\begin{proof}
Without loss of generality, assume $\mathbf{w}_{L}=\boldsymbol{1}$,
since we can absorb the weights of the last layer into the $L-1$
weight layer, as we did in single hidden layer case (eq. \eqref{eq:MSE 2-layer - reduced}).
Fix an activation pattern $p\in\left\{ 1,...,P\right\} $ as defined
in the beginning of this appendix. Set
\begin{equation}
{\bf v}_{L-2}^{p\left(n\right)}\triangleq\left(\prod_{m=1}^{l}\mathrm{diag}\left(\boldsymbol{a}_{m}^{p\left(n\right)}\right)\mathbf{W}_{m}\right)\mathbf{x}^{\left(n\right)}\,,\,{\bf V}_{L-2}^{p}\triangleq\left[{\bf v}_{L-2}^{p\left(1\right)},\dots,{\bf v}_{L-2}^{p\left(N\right)}\right]\in\mathbb{R}^{d_{L-2}\times N}\,
\end{equation}
and
\begin{equation}
{\bf G}_{L-1}^{p}\triangleq\mathbf{A}_{L-1}^{p}\circ{\bf V}_{L-2}^{p}
\end{equation}
Note that, since the activation pattern is fixed, the entries of ${\bf G}_{L-1}^{p}$
are polynomials in the entries of $(\mathbf{X},\boldsymbol{\mathcal{E}}_{1},\dots,\boldsymbol{\mathcal{E}}_{L-1},{\bf W}_{1},...,{\bf W}_{L-2})$,
and we may therefore apply Lemma \ref{lem: polynomial proof} to ${\bf G}_{L-1}^{p}$.
Thus, to establish $\mathrm{rank}({\bf G}_{L-1}^{p})=N$ $(\mathbf{X},\boldsymbol{\mathcal{E}}_{1},\dots,\boldsymbol{\mathcal{E}}_{L-1},{\bf W}_{1},...,{\bf W}_{L-2})$-a.e.,
we only need to exhibit a single $(\mathbf{X^{\prime}},\boldsymbol{\mathcal{E}}_{1}^{\prime},\dots,\boldsymbol{\mathcal{E}}_{L-1}^{\prime},{\bf W}_{1}^{\prime},...,{\bf W}_{L-2}^{\prime})$
for which $\mathrm{rank}({\bf G}_{L-1}^{\prime p})=N$. We note that
for a fixed activation pattern, we can obtain any value of $({\bf A}_{1}^{p},...,{\bf A}_{L-1}^{p})$
with some choice of $(\boldsymbol{\mathcal{E}}_{1}^{\prime},\dots,\boldsymbol{\mathcal{E}}_{L-1}^{\prime})$,
so we will specify $({\bf A}_{1}^{\prime p},...,{\bf A}_{L-1}^{\prime p})$
directly. We make the following choices:
\begin{gather}
x_{i}^{\prime\left(n\right)}=1\:\forall i,n\ ,\ a_{i,l}^{\prime p\left(n\right)}=1\,\forall l<L-2\ ,\forall i,n\\
{\bf W}_{l}^{\prime}=\left[{\bf 1}_{d_{l}\times1}\,,\,{\bf 0}_{d_{l}\times\left(d_{l-1}-1\right)}\right]\ \forall l\leq L-2\\
{\bf A}_{L-2}^{\prime p}=\left[{\bf 1}_{1\times d_{L-1}}\otimes\mathbf{I}_{d_{L-2}}\right]_{1,...,N}\:,\ {\bf A}_{L-1}^{\prime p}=\left[\mathbf{I}_{d_{L-1}}\otimes{\bf 1}_{1\times d_{L-2}}\right]_{1,...,N}
\end{gather}
where ${\bf 1}_{a\times b}$ (respectively, ${\bf 0}_{a\times b}$)
denotes an all ones (zeros) matrix of dimensions $a\times b$, $\mathbf{I}_{a}$
denotes the $a\times a$ identity matrix, and $[{\bf M}]_{1,...,N}$
denotes a matrix composed of the first $N$ columns of ${\bf M}$.
It is easy to verify that with this choice, we have $\mathbf{W}_{L-2}^{\prime}\left(\prod_{m=1}^{L-3}\mathrm{diag}\left(\boldsymbol{a}_{m}^{\prime p\left(n\right)}\right)\mathbf{W}_{m}^{\prime}\right)\mathbf{x}^{\prime\left(n\right)}={\bf 1}_{d_{L-2}\times1}$
for any $n$, and so ${\bf V}_{L-2}^{\prime p}={\bf A}_{L-2}^{\prime p}$
and
\begin{equation}
{\bf G}_{L-1}^{\prime p}=\mathbf{A}_{L-1}^{\prime p}\circ{\bf A}_{L-2}^{\prime p}=\left[\mathbf{I}_{d_{L-2}d_{L-1}}\right]_{1,...,N}
\end{equation}
which obviously satisfies $\mathrm{rank}({\bf G}_{L-1}^{\prime p})=N$.
We conclude that $\mathrm{rank}({\bf G}_{L-1}^{p})=N$, $(\mathbf{X},\boldsymbol{\mathcal{E}}_{1},\dots,\boldsymbol{\mathcal{E}}_{L-1},{\bf W}_{1},...,{\bf W}_{L-2})$-a.e.,
and remark this argument proves Fact \ref{fact:Khatri-Rao multiplies rank},
if we specialize to $L=2$.

As we did in the proof of Lemma \ref{lem:  when G full rank?}, we
apply the above argument for all values of $p$, and conclude via
Fact \ref{fact:intersect-a.e.} that $\mathrm{rank}({\bf G}_{L-1}^{p})=N$
for every $p$, $(\mathbf{X},\boldsymbol{\mathcal{E}}_{1},\dots,\boldsymbol{\mathcal{E}}_{L-1},{\bf W}_{1},...,{\bf W}_{L-2})$-a.e..
Since for every ${\bf w}$, ${\bf G}_{L-1}={\bf G}_{L-1}^{p}$ for
some $p$ which depends on ${\bf w}$, this implies that, $(\mathbf{X},\boldsymbol{\mathcal{E}}_{1},\dots,\boldsymbol{\mathcal{E}}_{L-1},{\bf W}_{1},...,{\bf W}_{L-2})$-a.e.,
$\mathrm{rank}({\bf G}_{L-1})=N$ simultaneously for all values of
${\bf W}_{L-1}$. Thus in any DLM of the MSE, with all weights except
$\mathbf{W}_{L-1}$ fixed, we can use eq. \ref{eq: Gme=00003D0} ($\mathbf{G}_{L-1}\mathbf{e}=0$),
and get $\mathbf{e}=0$.
\end{proof}

\end{document}